\documentclass[journal]{IEEEtran}
\usepackage{subfigure}
\usepackage{times}
\usepackage{epsfig}
\usepackage{color}
\usepackage{amsthm,amsfonts}
\usepackage{makecell}
\usepackage{amsmath,amsfonts}
\usepackage{algorithmic}
\usepackage{algorithm}
\usepackage{array}
\usepackage{textcomp}
\usepackage{stfloats}
\usepackage{url}
\usepackage{verbatim}
\usepackage{graphicx}
\usepackage{cite}
\usepackage{multirow}
\usepackage[caption=false,font=normalsize,labelfont=sf,textfont=sf]{subfig}
\newtheorem{theorem}{Theorem}
\newtheorem{Lemma}{Lemma}
\usepackage[capitalize]{cleveref}
\crefname{section}{Sec.}{Secs.}
\Crefname{section}{Section}{Sections}
\Crefname{table}{Table}{Tables}
\crefname{table}{Tab.}{Tabs.}
\usepackage{multirow, booktabs}
\begin{document}

\title{One for all: A novel Dual-space Co-training baseline for Large-scale Multi-View Clustering}

\author{Zisen Kong, Zhiqiang Fu, Dongxia Chang, Yiming Wang, Yao Zhao, \IEEEmembership{Fellow, IEEE}

\thanks{This work was supported in part by the National Natural Science Foundation of China under Grant 62272035, in part by the National Key Research and Development of China (No. 2018AAA0102100), and part by the National Natural Science Foundation of China (No. U1936212).

Z. Kong, Z. Fu, D. Chang, and Y. Zhao are with the Institute of Information Science, Beijing Jiaotong University, Beijing 100044, and also with the Beijing Key Laboratory of Advanced Information Science and Network Technology, Beijing 100044, China (email: zskong@bjtu.edu.cn; 
fuzhiqiang1230@outlook.com;
dxchang@bjtu.edu.cn; yzhao@bjtu.edu.cn).

Y. Wang is with the School of Computer Science, Nanjing University of Posts and Telecommunications, Nanjing, China (email: ymwang@njupt.edu.cn)}}

\markboth{Journal of \LaTeX\ Class Files,~Vol.~14, No.~8, August~2021}%
{Shell \MakeLowercase{\textit{et al.}}: A Sample Article Using IEEEtran.cls for IEEE Journals}


\maketitle

\begin{abstract}
In this paper, we propose a novel multi-view clustering model, named Dual-space Co-training Large-scale Multi-view Clustering (DSCMC). The main objective of our approach is to enhance the clustering performance by leveraging co-training in two distinct spaces. In the original space, we learn a projection matrix to obtain latent consistent anchor graphs from different views. This process involves capturing the inherent relationships and structures between data points within each view. Concurrently, we employ a feature transformation matrix to map samples from various views to a shared latent space. This transformation facilitates the alignment of information from multiple views, enabling a comprehensive understanding of the underlying data distribution. We jointly optimize the construction of the latent consistent anchor graph and the feature transformation to generate a discriminative anchor graph. This anchor graph effectively captures the essential characteristics of the multi-view data and serves as a reliable basis for subsequent clustering analysis. Moreover, the element-wise method is proposed to avoid the impact of diverse information between different views. Our algorithm has an approximate linear computational complexity, which guarantees its successful application on large-scale datasets. Through experimental validation, we demonstrate that our method significantly reduces computational complexity while yielding superior clustering performance compared to existing approaches. 
\end{abstract}

\begin{IEEEkeywords}
Article submission, IEEE, IEEEtran, journal, \LaTeX, paper, template, typesetting.
\end{IEEEkeywords}

\section{Introduction}
\IEEEPARstart{W}{ith} the advancement of information technology, data now originates from various sources and can be represented by diverse attributes, enabling the acquisition of data from multiple perspectives. However, directly integrating information between these views presents challenges due to their inherent heterogeneity \cite{ShiNWL23}. As a result, extracting valuable information from these diverse views becomes a significant concern. To address this, Multi-View Clustering (MVC) has emerged as a widely employed unsupervised data mining technique \cite{LuYL16, TrostenLJK21, Wen0X0HF023, YanZLT0LL23}.

In general, the existing multi-view clustering algorithms can be classified into four major categories, namely multi-view subspace clustering \cite{CaoZFLZ15, ShiNWL22, ZhangHFZC17}, graph-based clustering methods \cite{GMC, TanLHF023, ZhanNCNZY19}, matrix factorization methods \cite{GaoHLW13, zhao2017multi, ChenLW0L22}, and anchor graph-based clustering methods \cite{LiuW00LZG22, Qiang00021, WangLLJTZZ22}.

Specifically, the subspace self-representation theory illustrates that samples within the same subspace can be linearly represented by other samples. Leveraging this idea, the multi-view subspace clustering approach utilizes the dataset itself as a dictionary to learn similarity graphs that reflect the underlying subspace structure of the samples. For instance, in the work of \cite{WangLWZZH15}, a more effective similarity matrix is obtained by maximizing the correlation among samples in the same subspace across different views while minimizing the correlation of samples from different subspace views. On the other hand, graph-based multi-view clustering explores relationships between different samples independently to construct individual graphs for each view. Subsequently, it aims to find a fused graph incorporating information from various views. Inspired by this concept, \cite{GMC} proposed a method that generates a unified graph structure through graph embedding, enabling the integration of information from multiple views.

While the previously mentioned multi-view clustering methods have found wide-ranging applications, they encounter challenges when dealing with high-dimensional data and large sample sizes. Subspace and graph-based clustering approaches suffer from high computational and space complexity, limiting their effective application on large-scale datasets. To address this, researchers have proposed matrix factorization-based methods as an alternative solution. The essence of matrix factorization-based multi-view clustering is to decompose the original matrix into smaller dimensional basis matrices and coefficient matrices. One such approach is the nonnegative and orthogonal factorization method (e.g., \cite{YangZNWYW21}), which transforms the matrix factorization optimization problem into smaller-scale subproblems, making it more suitable for handling large-scale datasets.

In recent years, addressing the challenges posed by large-scale data in multi-view clustering has gained interest. One emerging approach is based on the concept of anchors. These methods select essential anchor points from the sample data and construct graphs based on the similarity between the sample points and these anchors, reducing computational and space complexity for large-scale datasets. One such method is scalable multi-view subspace clustering with unified anchors (SMVSC), which integrates anchor learning with graph construction for consistent graphs \cite{SMVSC}. However, SMVSC might face challenges in effectively learning a consistent graph, as it may not fully consider the complementary information between different views.

In this paper, we propose the DSCMC (Dual-Space Co-training Large-scale Multi-View Clustering) algorithm to overcome existing limitations in large-scale data clustering. It utilizes a projection matrix for a discriminative anchor graph in the original feature space and a transformation matrix for a low-dimensional latent space to ensure consistent samples with the learned anchor graph. By combining these components, our algorithm effectively captures consistency and complementarity from different views. Unlike other anchor-based methods \cite{FPMVS, OMSC}, we use an element-wise approach, setting all views to the same weight, which enhances robustness in ambiguous semantic information across views \cite{HuLY22, chen2023fast}. Extensive experiments on nine benchmark datasets demonstrate the superiority of our method over state-of-the-art approaches. In summary, our contributions are as follows:
\begin{itemize}
\item DSCMC effectively captures both complementary and consistent information from different views by learning features in both the original and latent spaces, resulting in a more discriminative anchor graph.
\item The element-wise method (instead of the common view-wise method) is used to avoid the impact of diverse information between different views on the clustering performance.
\item The proposed optimization strategy can be guaranteed to be performed within linear operational complexity and can be successfully applied to large-scale data clustering.
\item Experiments demonstrate the effectiveness of our DSCMC, showing that our proposed model outperforms state-of-the-art algorithms.
\end{itemize}

\section{Related Work}
In this section, we first introduce the notations used in the paper. Then two types of methods that are most relevant to our model are presented separately, namely multi-view matrix factorization-based clustering and anchor graph-based clustering method.
\subsection{Notations}
In this paper, $n$, $V$, and $k$ are used to represent the number of samples, views, and clusters, respectively. The multi-view data for different views is denoted as $X^v \in \mathbb{R}^{d_v \times n}$, where $d_v$ is the dimension of the $v$-th view, and the sum of dimensions across all views is denoted as $c$. The Frobenius norm and the $\ell_{2,1}$ norm of the matrix $Z$ are represented as $||Z||_F$ and $||Z||_{2,1}$, respectively. The transpose of matrix $X$ is denoted by $X^T$, and $I_m$ represents the identity matrix of dimension $m$.
\subsection{Matrix factorization-based MVC}
For multi-view dataset $X^v$, the purpose of matrix factorization (MF) is to decompose the original data into low-dimensional basis and coefficient matrices, which can be described as
\begin{equation}
\begin{aligned}
\min_{U^v\ge 0,U\ge0,H^v\ge 0}\sum_{v=1}^V\|X^v-H^vU^v\|_F^2+\lambda\|U-U^v\|_F^2
\label{MF}
\end{aligned}
\end{equation}
where $H^v$ is the basis matrix for $v$-th view, $U$ is the consistent coefficient matrix and $U^v$ is the coefficient matrix of the different views. This method based on matrix factorization is intuitive and effective. In addition, it also has certain advantages in computational time. However, its non-negative constraints on the base matrix and coefficient matrix are too strong, which may affect the determination of the coefficient matrix. Based on this, \cite{OPMC} proposes a one-step decomposition strategy that can directly obtain the clustering labels while removing the non-negativity constraint. The model can be written as
\begin{equation}
\begin{aligned}
&\min_{Y,C^v,U^v}\frac{1}{V}\sum_{v=1}^V\|X^v-YC^vU^v\|_F^2\\
&s.t. U^vU{^v}^T=I,Y_{ij}\in\{0,1\},\sum_{i=1}^kY_{ij}=1
\label{MF_OPMC}
\end{aligned}
\end{equation}
where $C^v$ is a centroid matrix, and $Y$ is the clustering indicator matrix.
\begin{figure*}[!ht]
  \centering
   \includegraphics[width=0.8\linewidth]{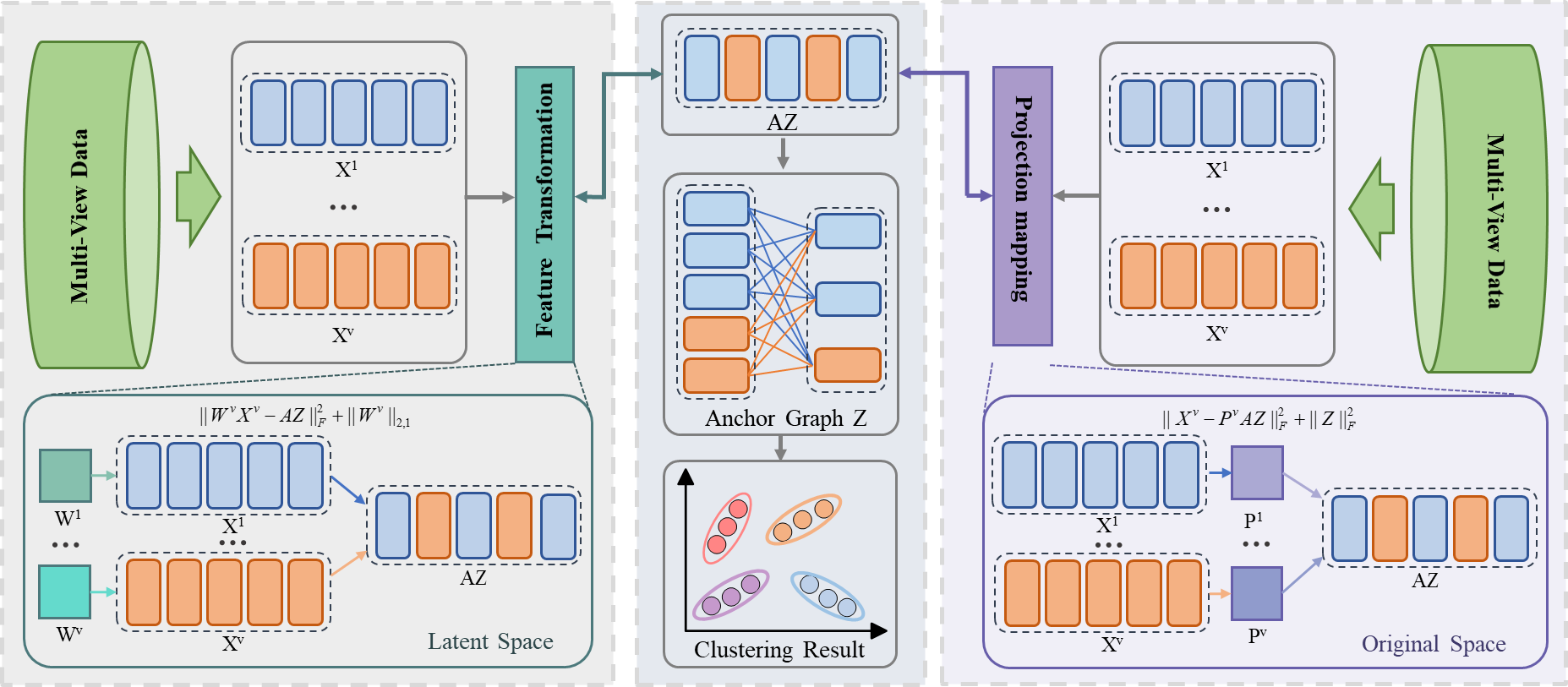}
   \caption{The framework of our DSCMC method. Our algorithm learns the consistent anchor graph collaboratively through the original space and the latent space.}
   \label{fig1}
\end{figure*}

\subsection{Anchor graph-based method}
The MF approach may be limited when the dimension and sample size become greater \cite{YangZNWYW21}. In order to reduce the significant impact of complexity on memory and computational speed, the anchor graph-based method has been proposed. Some early methods used heuristic sampling strategies, such as $k$-means or random sampling. However, this anchor selection and graph construction method are not unified, which may reduce clustering performance. In order to solve this problem, a method called fast parameter free multi-view subspace clustering with consumption anchor guidance (FPMVS) \cite{FPMVS} was proposed. FPMVS attempts to jointly optimize anchor point learning and graph construction to learn consistent graph structures of higher quality as follows
\begin{equation}
\begin{aligned}
&\min_{A,P^v,Z,\alpha_v}\sum_{v=1}^V\alpha_v^2\|X^v-P^vAZ\|_F^2\\
&s.t. \alpha_v^T\mathbf{1}=1,P{^v}^TP^v=I_k,A^TA=I_k,Z\ge 0,Z^T\mathbf{1}=1
\label{anchor_FPMVS}
\end{aligned}
\end{equation}
where $\alpha_v$ is the weight factor of $v$-th view, $P^v$ is the projection matrix of the different views, which can learn the consistent graph structure. $A$ is the anchor matrix and $Z$ is the consist anchor graph. 

The above-mentioned methods have been widely used in large-scale clustering algorithms and have achieved satisfactory results. However, due to the complexity of the data, it is difficult to fully guarantee the orthogonality or independence of the original space \cite{XieYYLY18}. Therefore, the learned consistent anchor graph may not be optimal.

\section{Methodology}
It is well known that the key to the success of multi-view learning lies in fully exploiting the complementarity and consistency among views, which are not available in single-view data. While exploring view information, the local structure of the data itself should be fully considered. In fact, the local structure of the data should be consistent in the original and latent spaces, which has a significant impact on the learning of discriminative anchor graphs. To this end, we propose a dual-space co-training large-scale multi-view clustering (DSCMC) algorithm. The method captures the complementary and consistent information of views while maintaining the local structure in different spaces.

\subsection{DSCMC: Formulation}
As previously analyzed, in order to adequately capture view information and maintain the local structure of the data, an effective co-training model is proposed, whose framework is shown in Figure \ref{fig1}. Specifically, in order to obtain complementary information between different views, we introduce a projection matrix $P^v$, which contains heterogeneous information from each view and can guide the latent representation $AZ$ to project back to the original space. Besides, the feature transformation matrix $W^v$ is proposed to map samples to latent space, reduce noise points and outlier interference, and capture consistent information within the view. 

The above two processes optimize each other for co-training and maintain the local structure of the different spaces. Overall, the objective function of our model can be expressed as 
\begin{equation}
\begin{aligned}
&\min_{W^v,P^v,A,Z}\underbrace{\sum_{v=1}^V\|X^v-P^vAZ\|_F^2}_{\text{Complementary\ information}}+\lambda_2\|Z\|_F^2\\
&+\lambda_1\underbrace{\sum_{v=1}^V\|W^vX^v-AZ\|_F^2}_{\text{Consistent\ information}}
+\lambda_3\|W^v\|_{2,1}, \\ 
&s.t. P{^v}^TP^v=I_k, A^TA=I_k, Z\ge 0,Z^T\mathbf{1}=1
\label{Obj_ours}
\end{aligned}
\end{equation}
where $P^v$ is the projection matrix of the different views. $A$ is the anchor matrix. To ensure independence and avoid trivial solutions, we impose an orthogonality constraint, limiting interactions between vectors in the matrix $P^v$ and $A$. $Z$ is the learned latent consistent graph and $W^v$ is used to choose distinguishing features. The $\ell_{2,1}$ norm is proposed to induce structured sparsity, enabling the selection of discriminative feature transformations to the latent space. $\lambda_1,\lambda_2,\lambda_3$ are the hyperparameters.

Next, we present the algorithmic optimization process of our DSCMC.

\subsection{DSCMC: Optimization}
Since the proposed algorithm is difficult to solve directly, an alternating iteration optimization method has been proposed for solving it. Using this iterative strategy, the original problem can be decomposed into the following subproblems.

\textbf{Update $P^v$:} Fixing $W^v$, $A$, $Z$, then $P^v$ can be updated by solving the following problem
\begin{equation}
\begin{aligned}
\min_{P^v}\sum_{v=1}^V\|X^v-P^vAZ\|_F^2 \quad s.t.P^v{^T}P^v=I_k
\label{update_P}
\end{aligned}
\end{equation}

Considering that $P^v$ is independent for each view, it can be solved separately. From the definition of the Frobenius norm, we can rewrite Eq.(\ref{update_P}) as
\begin{equation}
\begin{aligned}
&\min_{P^v}Tr(X{^v}^TX^v-2P{^v}^TX^vZ^TA^T+Z^TZ) \\ &s.t.P^v{^T}P^v=I_k
\label{update_P1}
\end{aligned}
\end{equation}

Since the solution to the subproblem is only related to $P^v$, the optimization problem shown in Eq.(\ref{update_P1}) can be further simplified to
\begin{equation}
\begin{aligned}
&\max_{P^v}Tr(2P{^v}^TX^vZ^TA^T) \\ &s.t.P^v{^T}P^v=I_k
\label{update_P2}
\end{aligned}
\end{equation}
Eq.(\ref{update_P2}) is an Orthogonal Procrustes Problem (OPP) that can be solved by singular value decomposition (SVD) \cite{WenHFFYZ19}. Assuming that the singular value of $X^vZ^TA^T$ is $U_P\Sigma_P V_P^T$, then we can obtain the result of $P^v$ by $U_PV_P^T$.

\textbf{Update $A$:} Once $P^v, W^v, Z$ are fixed, $A$ can be updated by the following subproblem
\begin{equation}
\begin{aligned}
&\min_{A}\sum_{v=1}^V\|X^v-P^vAZ\|_F^2+\lambda_1\sum_{v=1}^V||W^vX^v-AZ||_F^2 \\
&s.t.A{^T}A=I_k
\label{update_A}
\end{aligned}
\end{equation}

Similar to solve for $P^v$, Eq.(\ref{update_A}) can be formulated as
\begin{equation}
\begin{aligned}
&\max_{A}Tr(A^T(P{^v}^TXZ^T+\lambda_1W^vXZ^T) )\\ &s.t.A{^T}A=I_k
\label{update_A1}
\end{aligned}
\end{equation}
The solution to $A$ has a closed-form solution, that is, $A = U_AV_A^T$, where $U_A$ and $V_A$ represent the left and right singular value operators after SVD of $P{^v}^TX^vZ^T+\lambda_1W^vX^vZ^T$ respectively.

\textbf{Update $W^v$:} When $P^v$, $A$, $Z$ are all fixed, $W^v$ can be obtained by
\begin{equation}
\begin{aligned}
&\min_{W^v}\lambda_1\sum_{v=1}^V\|W^vX^v-AZ\|_F^2+\lambda_3\|W^v\|_{2,1}
\label{update_W}
\end{aligned}
\end{equation}
By the definition of the $\ell_{2,1}$ norm, we can obtain 
\begin{equation}
\begin{aligned}
\|W^v\|_{2,1}=Tr(W{^v}^T\Phi^v W^v)
\label{update_W1}
\end{aligned}
\end{equation}
where $\Phi^v$ denotes diagonal matrix of $W^v$. Therefore, Eq.(\ref{update_W}) can be rewritten as
\begin{equation}
\begin{aligned}
&\min_{W^v}\lambda_1\sum_{v=1}^V\|W^vX^v-AZ\|_F^2+\lambda_3Tr((W{^v}^T\Phi^v W^v)
\label{update_W2}
\end{aligned}
\end{equation}
This problem can be directly derived as
\begin{equation}
\begin{aligned}
W^v=\lambda_1AZX{^v}^T(\lambda_1X^vX{^v}^T+\lambda_3 \Phi^v)^{-1}
\label{update_W3}
\end{aligned}
\end{equation}

\textbf{Update $Z$:} $Z$ can be obtained by fixing other variables to solve the following problem
\begin{equation}
\begin{aligned}
&\min_{Z}\sum_{v=1}^V\|X^v-P^vAZ\|_F^2+\lambda_1\sum_{v=1}^V\|W^vX^v-AZ\|_F^2\\
&+\lambda_2\|Z\|_F^2 \quad s.t. Z\ge 0, Z^T\mathbf{1}=1
\label{update_Z}
\end{aligned}
\end{equation}
In fact, Eq.(\ref{update_Z}) is a classical quadratic programming (QP) problem \cite{FPMVS} and can be written as 
\begin{equation}
\begin{aligned}
&\min_{Z}\frac{1}{2}Z_{:,j}^TH_ZZ_{:,j}+F_Z^TZ_{:,j}\\
&s.t. Z\ge 0, Z^T\mathbf{1}=1
\label{update_Z1}
\end{aligned}
\end{equation}
where $Z_{:,j}$ is the $j$-th column of $Z$, $H_Z=2(V+\lambda_1+\lambda_2) I_m$, and $F_Z=-2\sum_{v=1}^V\lambda_1X_{:,j}{^v}^TW{^v}^TA-2\sum_{v=1}^VX_{:,j}{^v}^TP^vA$. This quadratic programming problem can be solved using the quadprog toolbox of Matlab.

The optimal solution process of the algorithm is summarised in Algorithm \ref{alg:algorithm}. After obtaining $Z$, we perform an SVD on $Z$ to obtain its right singular vectors $M$ and then perform $k$-means clustering on $M$.
\renewcommand{\algorithmicrequire}{ \textbf{Input:}} 
\renewcommand{\algorithmicensure}{ \textbf{Output:}} 
\begin{algorithm}[htbp]
\caption{The procedure of DSCMC}
\label{alg:algorithm}
\begin{algorithmic}[1] 
\REQUIRE Multi-view data $X^v$ , cluster number $k$, and parameters $\lambda_1,\lambda_2,\lambda_3$\\
\ENSURE Applying $k$-means to $M$
\STATE Initialize on $W^v$, $P^v$ $A$, $Z$.
\WHILE{not converged}
\STATE Update $P^v$ via Eq. (\ref{update_P2}).
\STATE Update $A$ via Eq. (\ref{update_A1}).
\STATE Update $W^v$ via Eq. (\ref{update_W3}).
\STATE Update $Z$ via Eq. (\ref{update_Z1}).
\ENDWHILE
\STATE \textbf{Return} The right singular vectors $M$ by performing SVD on $Z$
\end{algorithmic}
\end{algorithm}

\begin{figure*}
  \centering
  \subfigure[WebKB-4]{
	\begin{minipage}[t]{0.3\linewidth}
	\centering
	\includegraphics[width=1.8in]{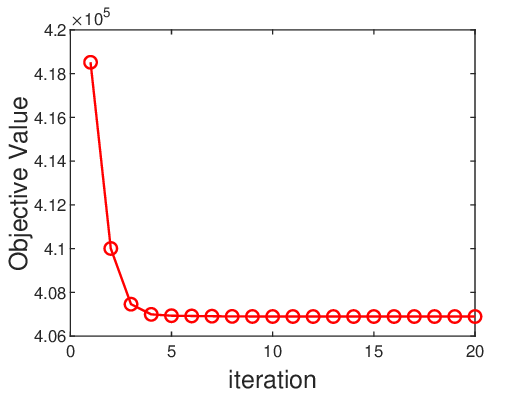}
\end{minipage}
   }%
  \subfigure[Caltech101-20]{
	\begin{minipage}[t]{0.3\linewidth}
	\centering
	\includegraphics[width=1.8in]{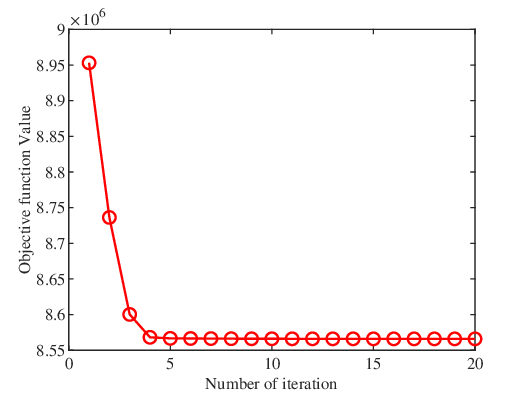}
\end{minipage}
   }%
  \subfigure[YoutubeFace]{
	\begin{minipage}[t]{0.3\linewidth}
	\centering
	\includegraphics[width=1.8in]{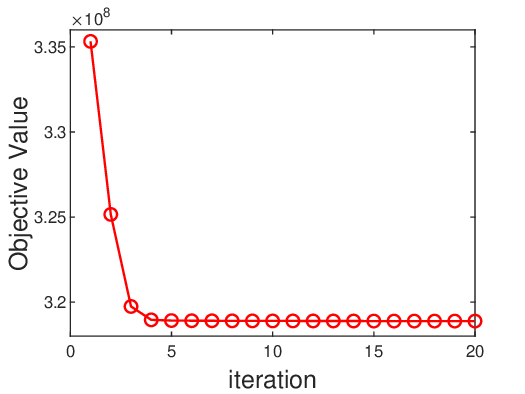}
\end{minipage}
   }%
  \caption{Convergence of three benchmark datasets.}
  \label{Fig.lable}
\end{figure*}

\subsection{Computational Complexity}
Our model consists of four main components. For $P^v$, the iterative computation consists mainly of two parts, i.e. SVD decomposition and matrix multiplication, which has an operational complexity of $\mathcal{O}(ck^2)$ and $\mathcal{O}(ckn)$. Similarly, the complexity required to update variable $A$ are $\mathcal{O}(2k^2n)$ and $\mathcal{O}(2ck^2)$. Besides, the time complexity of update $W^v$ requires $\mathcal{O}(ckn+ck^2)$. And the optimization of $Z$ is a quadratic programming (QP) problem with a total computational complexity of $\mathcal{O}(k^3n)$. In summary, the operational complexity of DSCMC is $\mathcal{O}(4ck^2+2ckn+2k^2n+k^3n)$. Normally, the original data satisfies $k \ll n$, and $c \ll n$, the computational complexity of the DSCMC algorithm is almost linearly related to the number of samples, that is $\mathcal{O}(n)$.

\subsection{Convergence analysis}
The objective function of DSCMC involves four variables, making it challenging to directly prove its strong convergence. The algorithm employs alternating iterations, updating one variable while keeping the others fixed. To demonstrate convergence, we plot the relationship between the number of iterations and the objective function in Figure \ref{Fig.lable}. The objective function value, denoted as $Obj$, is calculated using the expression: $Obj = (\|X^v - P^vAZ\|_F^2 + \lambda_1\|W^vX^v - AZ\|_F^2 + \lambda_2\|Z\|_F^2 + \lambda_3\|W^v\|_{2,1})$. As shown in Figure \ref{Fig.lable}, the objective function value monotonically decreases and converges within approximately 10 iterations. More theoretical proofs are available in the supplementary material. 

\subsection{Links with other methods}
In this section, the connections and differences between the proposed DSCMC and the two most similar algorithms (i.e. SMVSC and OMSC) are explored.

\textbf{Connections to SMVSC:} The objective function of SMVSC is shown as Eq.(\ref{SMVSC})
\begin{equation}
\begin{aligned}
&\min_{A,P^v,Z,\alpha_v}\sum_{v=1}^V\alpha_v^2\|X^v-P^vAZ\|_F^2+ \lambda\|Z\|_F^2 \\
&s.t. \alpha_v^T\mathbf{1}=1,P{^v}^TP{^v}=I,A^TA=I_m,Z\ge 0,Z^T\mathbf{1}=1
\label{SMVSC}
\end{aligned}
\end{equation}
Compared to SMVSC, our algorithm has many differences. Firstly, the feature transformation matrix $W^v$ is used to transform the data of the original space into the latent space, which improves the quality of the latent anchor graph. Then $\ell_{2,1}$ norm is proposed to learn a better anchor graph. Therefore, our algorithm outperforms SMVSC.

\textbf{Connections to OMSC:} The model of OMSC can be described as:
\begin{equation}
\begin{aligned}
&\min_{\alpha_v,P^v,A,Z,G,F,}\sum_{v=1}^V\alpha_v^2\|X^v-P^vAZ\|_F^2+\lambda\|Z-GF\|_F^2 \\
&s.t. \alpha_v^T\mathbf{1}=1,P{^v}^TP{^v}=I,A^TA=I,Z\ge 0,Z^T\mathbf{1}=1\\
&G^TG=I,F_{ij}\in \{0,1\},\sum_{i=1}^kF_{ij}=1, \ \forall j=1,2,...,n
\label{OMSC}
\end{aligned}
\end{equation}
As illustrated in Eq.(\ref{OMSC}), OMSC incorporates partition information and graph construction into a unified framework. If we make $F_{ij}=0$, OMSC will degrade to SMVSC. In contrast, we use the projection matrix $P^v$ and the feature transformation matrix $W^v$ for co-training, which allows the learned semantic information to be effectively represented in both the original space and the latent space, resulting in better features. In addition, the element-wise strategy is proven to perform better when the view clustering structure is not clear. Hence, the clustering performance of DSCMC is better than OMSC.

\section{Experiments and analysis}
In this section, we conduct experiments on nine benchmark datasets to validate the performance of the proposed algorithm. In the experiments, we compare DSCMC with nine state-of-the-art algorithms. Furthermore, we conduct an analysis of the construction of the complete graph and provide visual representations. Lastly, ablation experiments are performed to assess the effectiveness of our proposed algorithm.

\subsection{Benchmark datasets}
In the experiments, we use nine commonly used benchmark datasets: 3-sources, WebKB-4, Caltech101-7, Caltech101-20, BDGP, NUSWIDE, VGGFace2-50, CIFAR100, and YoutubeFace. These datasets include text datasets, image datasets, and face datasets. The largest dataset contains over 120,000 samples. The specific characteristics of these datasets are shown in Table \ref{table1}.

\begin{table}[htbp]
\caption{Dataset Description}
\centering
\resizebox{\linewidth}{!}{
\begin{tabular}{cccccc}
\toprule
Type & Dataset & Sample & Class & View & Dimension \\
\midrule
\multirow{2}{*}{Texts} & 3-sources & 169 & 6 & 3 & 3560,3631,3068\\
 & WebKB-4 & 203 & 4 & 3 & 1703,230,230 \\
\cmidrule(lr){1-6}
\multirow{5}{*}{Images}& Caltech101-7 & 1474 & 7 & 6 & 48,40,254,1984,512,928\\
& Caltech101-20 & 2386 & 20 & 6 & 48,40,254,1984,512,928\\
& BDGP & 2500 & 5 & 2 & 1750,79 \\
& NUSWIDE & 30,000 & 31 & 5 & 65,226,145,74,129 \\
& CIFAR100 & 50,000 & 100 & 3 & 512,2048,1024 \\
\cmidrule(lr){1-6}
\multirow{2}{*}{Faces}& VGGFace2-50 & 34,027 & 50 & 4 & 944,576,512,640  \\
& YoutubeFace & 126,054 & 50 & 4 & 944,576,512,640 \\
\bottomrule
\end{tabular}}
\label{table1}
\end{table}

\subsection{Comparison algorithm}
We compare our algorithm with the following state-of-the-art methods: Diversity-induced multi-view subspace clustering (DiMSC) \cite{CaoZFLZ15}, Parameter-Free Auto-Weighted Multiple Graph Learning (AMGL) \cite{AMGL}, Latent Multi-view Subspace Clustering (LMSC) \cite{ZhangHFZC17}, Graph-Based Multi-View Clustering (GMC) \cite{GMC}, Multiview Clustering: A Scalable and Parameter-Free Bipartite Graph Fusion Method (SFMC) \cite{SFMC}, Large-scale Multi-view Subspace Clustering in Linear Time (LMVSC) \cite{LMVSC}, One-pass Multi-view Clustering for Large-scale Data (OPMC) \cite{OPMC}, Scalable Multi-view Subspace Clustering with Unified Anchors(SMVSC) \cite{SMVSC}, Fast Parameter-Free Multi-View Subspace Clustering With Consensus Anchor Guidance (FPMVS) \cite{FPMVS}, Efficient Orthogonal Multi-view Subspace Clustering (OMSC) \cite{OMSC}, Auto-weighted Multi-view Clustering for Large-scale Data (AWMVC) \cite{AWMVC}. Among them, SFMC, LMVSC, OPMC,  FPMVS, SMVSC, OMSC, and AWMVC are large-scale multi-view clustering methods.

\subsection{Experimental settings}
In our experiments, we set the initial values of $W^v$, $P^v$, $A$, and $Z$ to zero. The maximum number of iterations is fixed at 20. For algorithms requiring the use of $k$-means to obtain the final clustering results, we perform 50 runs to mitigate the randomness associated with $k$-means initialization. For the anchor matrix $A\in\mathbb{R}^{m \times m}$, we set the number of anchors to the number of clusters, i.e., $m=k$. The choice of the number of anchors is discussed in the supplementary material. All experiments were conducted on a desktop computer with an Intel(R) Core(TM) i9-13900K CPU @ 3.00 GHz and 64GB of RAM. The programming software employed was Matlab R2021a (64-bit). To assess performance, we employ four common clustering indicators: Accuracy (ACC), Normalized Mutual Information (NMI), F-score, and Adjusted Rand Index (ARI).

\subsection{Using $\ell_{2,1}$ norm for feature transformation}
In cluster analysis tasks, one needs to extract efficient and robust features \cite{NieHCD10, Fu0CZW21}. However, data from different classes have specific feature attributes and a particular feature may be important for one class but not for another. Therefore, we need to capture the most discriminative features of each class in order to separate them by category. 

Instead of using the common Frobenius norm in the feature transformation process, we introduced the $\ell_{2,1}$ norm in our model. This is because $\ell_{2,1}$ norm is more robust to outliers and provides a comprehensive understanding of the underlying data distribution. We conducted experiments on the Caltech101-20 dataset, where the Frobenius norm and $\ell_{2,1}$ norm are selected for visualization, respectively. The results show that $\ell_{2,1}$ norm can learn clearer structural features than the Frobenius norm. Proof of theory demonstrated in supplementary material.
\begin{figure}[htbp]
   \centering
       \subfigure[Frobenius norm]{
	\begin{minipage}[t]{0.45\linewidth}
		\centering
		\includegraphics[width=1.3in]{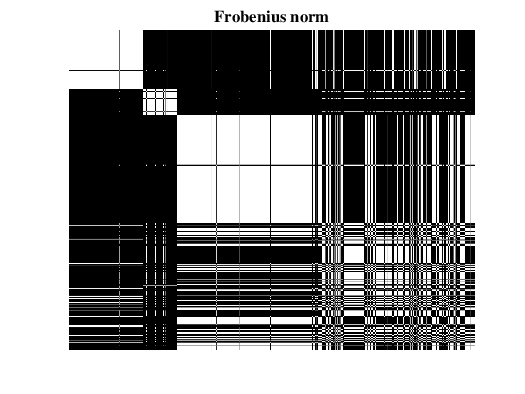}
		\end{minipage}
	}%
        \subfigure[$\ell_{2,1}$ norm]{
		\begin{minipage}[t]{0.45\linewidth}
			\centering
			\includegraphics[width=1.3in]{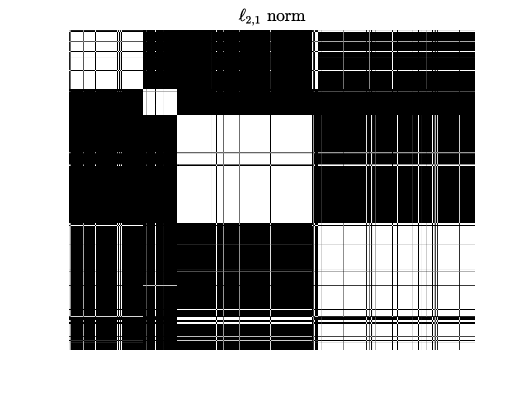}
		\end{minipage}
	}%
  \caption{Complete graph of the Caltech101-20 dataset compared at different norms.}
  \label{fig2}
\end{figure}

\begin{table*}[htbp]
\renewcommand\arraystretch{1.1}
\caption{Experimental results of different algorithms on nine datasets, where the "N/A" symbol indicates that the algorithm was unable to complete the computation due to memory constraints and the "-" symbol indicates that the solution does not exist or is not unique on the dataset.}
\centering
\resizebox{\linewidth}{!}{
\begin{tabular}
{ccccccccccccc}
\toprule
Database & DiMSC & AMGL & LMSC & GMC & SFMC & LMVSC  & OPMC & FPMVS & SMVSC  & OMSC & AWMVC & DSCMC\\
\midrule
\multicolumn{13}{c}{ACC}\\
\cmidrule(lr){1-13}
3-sources & \underline{0.7101} & 0.1865 & 0.7041 & 0.6923 & 0.3491 & 0.4970 & 0.5207 & 0.3491 & 0.6627 & 0.3373 & \textbf{0.7160} & \textbf{0.7160}\\
WebKB-4 & - & 0.3024 & 0.6847 & \underline{0.7586} & 0.5419 & 0.7192 & 0.6847 & 0.6404 & 0.7291 & 0.6502 & 0.6453 & \textbf{0.8128} \\
Caltech101-7 & 0.4389 & 0.3125 & 0.6092 & 0.6920 & 0.5651 & 0.3446 & 0.5475 & 0.6920 & \underline{0.7802} & 0.6683 & 0.4518 & \textbf{0.8616} \\
Caltech101-20 & 0.4388 & 0.3750 & 0.5000 & 0.4564 & 0.5947 & 0.4782 & 0.5448 & 0.6639 & 0.6442 & \underline{0.6676}  & 0.4845 & \textbf{0.7619}\\
BDGP & 0.7452 & 0.4474 & 0.4680 & 0.7324 & 0.3780 & \underline{0.7916}  & 0.4584 & 0.6460 & 0.6452 & 0.6008 & 0.4792 & \textbf{0.9456}\\
NUSWIDE & N/A & N/A & N/A & N/A & 0.1689 & 0.1495 & 0.1607 & 0.1944 & 0.1916 & \underline{0.1994}  & 0.1300 & \textbf{0.2054}\\
VGGFace2-50 & N/A & N/A & N/A & N/A & N/A & 0.1271 & 0.1198 & 0.1136 & 0.1218 & 0.1154 & \underline{0.1442} & \textbf{0.1461}\\
CIFAR100 & N/A & N/A & N/A & N/A & N/A & \underline{0.9232} & 0.8780 & 0.7528 & 0.7466 & 0.8700 & 0.9051 & \textbf{0.9307}\\
YoutubeFace & N/A  & N/A & N/A & N/A & N/A & \textbf{0.7515} & 0.7002 & 0.6904 & 0.6581 & 0.7152 & 0.7260 & \underline{0.7397}\\
\cmidrule(lr){1-13}
\multicolumn{13}{c}{NMI}\\
\cmidrule(lr){1-13}
3-sources & 0.6289 & 0.0802 & \underline{0.6618} & 0.6216  & 0.0821 & 0.4230 & 0.3706 & 0.1177 & 0.5423 & 0.1040 & 0.5956 & \textbf{0.6689}\\
WebKB-4 & - & 0.0576 & 0.3557 & 0.4219 & 0.0551 & \underline{0.4843} & 0.4018 & 0.2914 & 0.3781 & 0.2048 & 0.4107 & \textbf{0.4916} \\
Caltech101-7 & 0.4287 & 0.3615 & 0.6098 & \underline{0.6595} & 0.5626 & 0.1595 & 0.5062 & 0.5426 & 0.6082 & 0.5416 & 0.5304 & \textbf{0.6654}\\
Caltech101-20 & 0.5655 & 0.5818 & 0.5787 & 0.4809 & 0.5641 & 0.5751 & \underline{0.6880} & 0.6387 & 0.6203 & 0.6179 & 0.5304 & \textbf{0.6891}\\
BDGP & \underline{0.7622} & 0.2696 & 0.2486 & 0.7195 & 0.3519 & 0.6680 & 0.2887 & 0.4561 & 0.4614 & 0.4076 & 0.3006 & \textbf{0.8681}\\
NUSWIDE & N/A & N/A & N/A & N/A & 0.0601 & 0.1265 & \textbf{0.1555} & 0.1345 & 0.1287 & 0.1270 & 0.1217 & \underline{0.1413}\\
VGGFace2-50 & N/A & N/A & N/A & N/A & N/A & 0.1457 & 0.1453 & 0.1370 & 0.1427 & 0.1446 & \textbf{0.1641} & \underline{0.1622}\\
CIFAR100 & N/A & N/A & N/A & N/A & N/A & \underline{0.9872} & 0.9800 & 0.9176 & 0.9211 & 0.9792 & 0.9842 & \textbf{0.9885}\\
YoutubeFace  & N/A & N/A & N/A & N/A & N/A & 0.8392 & 0.8327 & 0.8408 & 0.8131 & 0.8527 & \underline{0.8551} & \textbf{0.8559}\\
\cmidrule(lr){1-13}
\multicolumn{13}{c}{Fscore}\\
\cmidrule(lr){1-13}
3-sources & 0.6033 & 0.2780 & 0.6490 & 0.6047  & 0.3873 & 0.4951 & 0.4687 & 0.2662 & 0.5991 & 0.2748 & \underline{0.6508} & \textbf{0.6790}\\
WebKB-4 & - & 0.3782 & 0.6249 & 0.6857 & 0.5620 & 0.6677 & 0.6706 & 0.5684 & \underline{0.6934} & 0.5681 & 0.5880 & \textbf{0.7404}\\
Caltech101-7 & 0.4430 & 0.1697 & 0.6115 & 0.7217 & 0.5855 & 0.3207 & 0.5430 & 0.7007 & \underline{0.7620} & 0.6201 & 0.4961 & \textbf{0.8659} \\
Caltech101-20 & 0.3734 & 0.0952 & 0.4186 & 0.3403 & 0.4303 & 0.3862 & 0.5142 & \underline{0.6917} & 0.6745 & 0.6768 & 0.4961 & \textbf{0.7622}\\
BDGP & \underline{0.7451} & 0.3925 & 0.3561 & 0.7063 & 0.4215 & 0.6282 & 0.3627 & 0.5112 & 0.5323 & 0.4780  & 0.4001 & \textbf{0.9002}\\
NUSWIDE & N/A & N/A & N/A & N/A & 0.1068 & 0.0947 & 0.1017 & \underline{0.1365} & 0.1321 & 0.1323  & 0.0823 & \textbf{0.2301}\\
VGGFace2-50 & N/A & N/A & N/A & N/A & N/A & 0.0606 & 0.0570 & 0.0588 & 0.0575 & 0.0602 & \textbf{0.0701} & \underline{0.0700}\\
CIFAR100 & N/A & N/A & N/A & N/A & N/A & \underline{0.9294} & 0.8978 & 0.7160 & 0.7307 & 0.8840 & 0.9190 & \textbf{0.9401}\\
YoutubeFace & N/A & N/A & N/A & N/A & N/A & 0.6617 & 0.6274 & 0.6439 & 0.5971 & \underline{0.6758} & 0.6710 &
\textbf{0.6807}\\
\cmidrule(lr){1-13}
\multicolumn{13}{c}{ARI}\\
\cmidrule(lr){1-13}
3-sources & 0.4959 & 0.1661 & 0.5591 & 0.4431  & 0.0231 & 0.3333 & 0.3345 & 0.0537 & 0.4835 & 0.0473  & \underline{0.5605} & \textbf{0.5958}\\
WebKB-4 & - & 0.0110 & 0.3910 & 0.4114 & 0.0161 & \underline{0.4956} & 0.4888 & 0.2975 & 0.4789 & 0.2796 & 0.3708 & \textbf{0.5193}\\
Caltech101-7 & 0.2876 & 0.0581 & 0.4696 & 0.5943 & 0.4116 & 0.0990 & 0.3884 & 0.5644 & \underline{0.6436} & 0.4396 & 0.3517 & \textbf{0.7677} \\
Caltech101-20 & 0.3154 & 0.0302 & 0.3542 & 0.1284 & 0.2765 & 0.3279 & 0.4599 & \underline{0.6360} & 0.6190 & 0.6164 & 0.3567 & \textbf{0.7074} \\
BDGP & \underline{0.6752} & 0.1600 & 0.1935 & 0.6152 & 0.1615 & 0.5326 & 0.1939 & 0.3778 & 0.4111 & 0.3393  & 0.2498 & \textbf{0.8752}\\
NUSWIDE & N/A & N/A & N/A & N/A & 0.0124 & 0.0196 & 0.0616 & \underline{0.0656} & 0.0654 & 0.0625  & 0.0409 & \textbf{0.0833}\\
VGGFace2-50 & N/A & N/A & N/A  & N/A & N/A & 0.0411 & 0.0369 & 0.0335 & 0.0366 & 0.0329  & \underline{0.0504} & \textbf{0.0508}\\
CIFAR100 & N/A & N/A & N/A & N/A & N/A & \underline{0.9286} & 0.8967 & 0.7126 & 0.7275 & 0.8827 & 0.9181 & \textbf{0.9395}\\
YoutubeFace & N/A & N/A & N/A  & N/A & N/A & 0.6538 & 0.6182 & 0.6348 & 0.5872 & \underline{0.6683} & 0.6632 & \textbf{0.6731}\\
\bottomrule
\end{tabular}}
\label{table2}
\end{table*}

\subsection{Clustering performance analysis}
We compared DSCMC with the state-of-the-art nine multi-view clustering algorithms on nine datasets. The experimental results are shown in Table \ref{table2}, where “N/A" represents out-of-memory. It is worth noting that bold and underlined indicate the best and the second-best performance respectively. From Table \ref{table2}, we have the following conclusions:

\begin{itemize}
\item In summary, DSCMC exhibits excellent clustering performance on nine commonly used benchmark datasets. Notably, in the text datasets 3-sources and WebKB-4, DSCMC outperforms most compared algorithms. In the image dataset Caltech101-7, DSCMC shows a significant improvement over the other methods (achieving an ACC value 10\% higher than the second-best SMVSC). Additionally, the performance of DSCMC is competitive on the face datasets VGGFace2-50 and YoutubeFace.
\item  Through a comparison between SMVSC and OMSC, it becomes evident that the element-wise method outperforms the view-wise approach, particularly when the clustering structure is not clear enough.
\item Compared to algorithms that perform feature learning only in the original feature space, such as LMVSC, FPMVS, SMVSC, and OMSC, our method adopts a co-training approach in both original and latent space. This strategy contributes to obtaining a higher-quality anchor graph, resulting in significant performance enhancements and improved clustering results.
\item Compared with DiMSC, AMGL, LMSC, and GMC, DSCMC is more advantageous in terms of computational efficiency and storage requirement, which is due to the fact that our method learns the similarity of the samples through the anchor graph, and the computational complexity is reduced to $\mathcal{O}(n)$. Therefore, DSCMC is very effective in dealing with large-scale data clustering tasks.
\end{itemize}

\subsection{Ablation experiment}
Our algorithm learns the discriminative anchor graph in both the original and latent spaces. To show its effectiveness, we designed two modules for comparison: one focused solely on the latent space (only $W^v$) and the other on the original space (only $P^v$). Additionally, we created modules DSCMC-$F$ and DSCMC-$\alpha$ to demonstrate the impact of the $\ell_{2,1}$ norm and element-wise method, respectively. The results of the ablation experiments are shown in Table \ref{table3}. These results clearly illustrate the decline in clustering performance when significant components are removed from DSCMC. Hence, our algorithmic framework effectively enhances the quality of the anchor graph, leading to improved clustering results. For further experimental details of the ablation experiment, refer to the supplementary material.
\begin{table}[!ht]
\renewcommand\arraystretch{1.3}
\caption{The ablation experiment of our algorithm on three benchmark datasets}
\begin{center}
\resizebox{1\linewidth}{!}{
\begin{tabular}{c|ccc|ccc|ccc}
\hline
\multirow{3}{*}{Model}  & \multicolumn{9}{c}{Datasets} \\[1pt]
\cline{2-10}
~  & \multicolumn{3}{c|}{3-sources} & \multicolumn{3}{c|}{Caltech101-20}& \multicolumn{3}{c}{BDGP}  \\
\cline{2-10}
~ & ACC & NMI & ARI & ACC & NMI & ARI & ACC & NMI & ARI\\
\hline
only $P^v$ & 0.6627 & 0.6213 &0.4741 & 0.6505 & 0.6250 & 0.5556 & 0.6108 & 0.4166 & 0.3657\\[1.2pt]
only $W^v$ & 0.6154 & 0.4326 & 0.4047 & 0.4493 & 0.3341 & 0.2475 & 0.7740 & 0.6923 & 0.6485\\[1.2pt]
DSCMC-$F$ & 0.4201 & 0.1203 & 0.0769 & \underline{0.6882} & \underline{0.6460} &  \underline{0.6320} & \underline{0.9452} & \underline{0.8670} & \underline{0.8744}\\
DSCMC-$\alpha$ & \underline{0.6982} & \underline{0.6574} & \underline{0.5719} & 0.5126 & 0.6183 & 0.3841 & 0.6452 & 0.4305 & 0.3711\\
Ours & \textbf{0.7160} & \textbf{0.6689} & \textbf{0.5958} & \textbf{0.7619} & \textbf{0.6891} & \textbf{0.7074} & \textbf{0.9456} & \textbf{0.8681} & \textbf{0.8752}\\
\hline
\end{tabular}}
\end{center}
\label{table3}
\end{table}

\subsection{Complete graph construction}
To further understand the performance of the algorithm, we visualize the complete graph $Z^TZ \in \mathbb{R}^{n \times n}$ on the Caltech101-20 dataset.

According to subspace theory, a sample may be closer to samples from the same subspace, and samples from different spaces cannot be represented by each other. Therefore, the ideal clustering structure should be close to the block diagonal structure \cite{FengLXY14}. As presented in Figure \ref{Fig3}, our complete graph exhibits a more distinct block structure and contains less redundancy.
\begin{figure}[htbp]
  \centering
  \subfigure[LMVSC]{
	\begin{minipage}[t]{0.28\linewidth}
	\centering
	\includegraphics[width=1.1in]{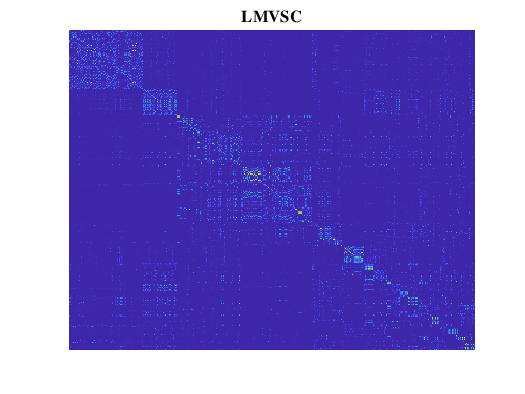}
\end{minipage}
}
  \subfigure[FPMVS]{
	\begin{minipage}[t]{0.28\linewidth}
	\centering
	\includegraphics[width=1.1in]{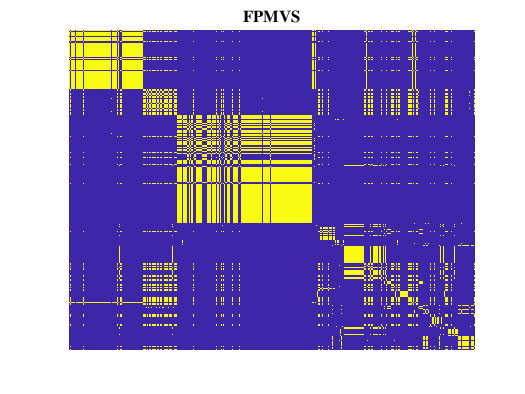}
\end{minipage}
}
  \subfigure[SMVSC]{
	\begin{minipage}[t]{0.28\linewidth}
	\centering
	\includegraphics[width=1.1in]{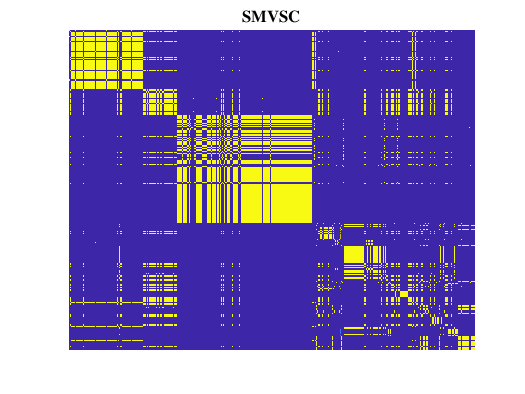}
\end{minipage}
}
  \subfigure[OMSC]{
	\begin{minipage}[t]{0.28\linewidth}
	\centering
	\includegraphics[width=1.1in]{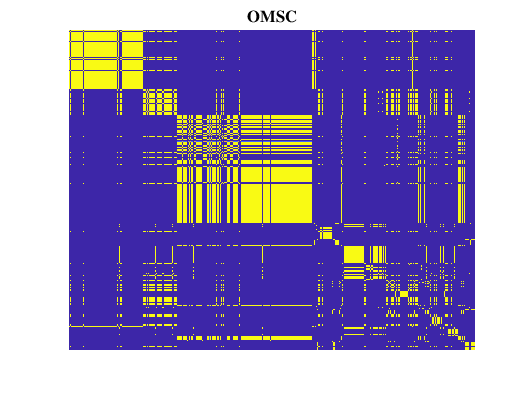}
\end{minipage}
}
  \subfigure[AWMVC]{
	\begin{minipage}[t]{0.28\linewidth}
	\centering
	\includegraphics[width=1.1in]{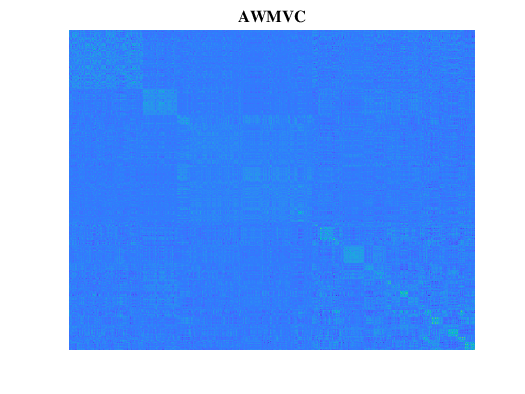}
\end{minipage}
}
  \subfigure[Ours]{
	\begin{minipage}[t]{0.28\linewidth}
	\centering
	\includegraphics[width=1.1in]{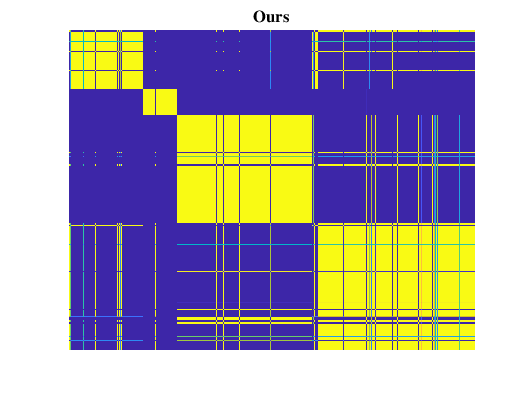}
\end{minipage}
}
  \caption{Comparison of complete graphs constructed by different algorithms.}
  \label{Fig3}
\end{figure}

\subsection{Parameter sensitivity and analysis}
As shown in Eq.(\ref{Obj_ours}), our algorithm contains three hyperparameters, $\lambda_1$, $\lambda_2$, and $\lambda_3$. In our experiments, the range of all three parameters in a wide range from $[10^{-3},10^{-2},10^{-1},10^{0},10^{1},10^{2},10^{3}]$. To test the sensitivity of the parameters, different combinations were performed. Firstly, we fixed $\lambda_3$ to tune the other two parameters, and the results are shown in Figure \ref{fig4} (a). Then, we fixed $\lambda_1$ and $\lambda_2$ to adjust $\lambda_3$, and the sensitivity can be seen in Figure \ref{fig4} (b). Figure \ref{fig4} implies that DSCMC can achieve better results on the NUSWIDE dataset when $\lambda_1 \le 10^0$, $\lambda_2 \le 10^0$, and $\lambda_3 \le 10^0$.
\begin{figure}[htbp]
\centering
\includegraphics[width=1\linewidth]{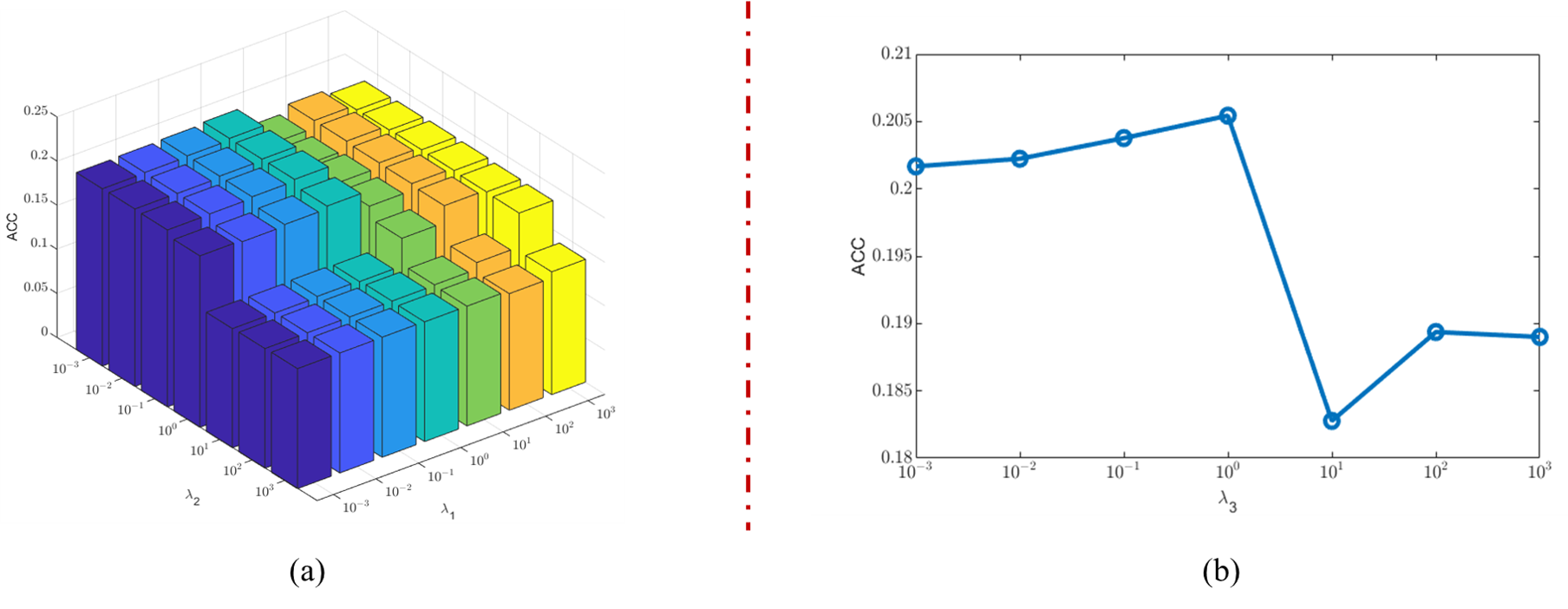}
\caption{Parameter sensitivity analysis of our method on NUSWIDE dataset, where (a) fix $\lambda_3$ to tune $\lambda_1$ and $\lambda_2$; (b)$\lambda_1$ and $\lambda_2$ are fixed and tune $\lambda_3$.}
\label{fig4}
\end{figure}

\section{Conclusion}
In this paper, we introduce a novel clustering approach called Dual-Space Co-training Large-scale Multi-view Clustering (DSCMC). Our method performs co-training in two distinct spaces to improve clustering performance. The method involves the use of a projection matrix to obtain the latent anchor graph from different views and the introduction of a transformation matrix to map data from the original space to the latent consistent space. By co-training these two spaces, the quality of anchor graphs is enhanced, resulting in better clustering performance. A notable advantage of DSCMC is the element-wise strategy, which helps mitigate the influence of diverse information between different views. Furthermore, DSCMC has a linear computational complexity, which makes it a highly efficient solution for clustering large-scale datasets. The experimental results consistently demonstrate the superiority of DSCMC over existing approaches in terms of clustering accuracy and scalability.

{\appendix[Proof of convergence]
The objective function of our model is shown below:
\begin{equation}
\begin{aligned}
\min_{W^v,P^v,A,Z}&\sum_{v=1}^V\|X^v-P^vAZ\|_F^2+\lambda_1\sum_{v=1}^V\|W^vX^v-AZ\|_F^2\\
&+\lambda_2\|Z\|_F^2+\lambda_3\|W^v\|_{2,1}, \\ 
&s.t. P{^v}^TP^v=I_k, A^TA=I_k, Z\ge 0,Z^T\mathbf{1}=1
\end{aligned}
\end{equation}
To prove the convergence of the objective function, the following theorem is granted.
\begin{theorem}
The objective function of our algorithm will be non-increasing for the iteration.
\label{theorem1}
\end{theorem}

Firstly, the objective function can be written as
\begin{equation}
\begin{aligned}
&\Gamma(\{\mathbf{P}^v\}_{v=1}^V, \{\mathbf{W}^v\}_{v=1}^V,A,Z)=\\
&\min_{W^v,P^v,A,Z}\sum_{v=1}^V\|X^v-P^vAZ\|_F^2+\lambda_1\sum_{v=1}^V\|W^vX^v-AZ\|_F^2\\
&+\lambda_2\|Z\|_F^2+\lambda_3\|W^v\|_{2,1} \\ 
\label{conv1}
\end{aligned}
\end{equation}

Next, we present the proof of Theorem \ref{theorem1}. 
\begin{proof}
For clarity, we firstly define the $t$-iteration symbols by $\{\mathbf{P_t^v}\}_{v=1}^V$, $ \{\mathbf{W_t^v}\}_{v=1}^V$, $A_t$, $Z_t$. To prove Theorem \ref{theorem1}, we introduce the Lemma \ref{lemma2}:
\begin{Lemma}
If there exist two positive constants $b$ and $c$, the following inequality holds.
\begin{equation}
\begin{aligned}
c^2-\frac{c^4}{2c^2}\ge b^2-\frac{b^4}{2c^2}
\end{aligned}
\end{equation}
\label{lemma2}
\end{Lemma}
\begin{proof}
For two positive constants $b$ and $c$, it is natural that 
\begin{equation}
\begin{aligned}
(b^2-c^2)^2\ge 0 \to (b^4-2b^2c^2+c^4)\ge 0\\
\Leftrightarrow  2c^4-c^4\ge 2b^2c^2-b^4\\
\Leftrightarrow c^2-\frac{c^4}{2c^2}\ge b^2-\frac{b^4}{2c^2}
\end{aligned}
\end{equation}
\end{proof}
1) Update $\{\mathbf{P_{t+1}^v}\}_{v=1}^V$ while fixing $\{\mathbf{W_{t}^v}\}_{v=1}^V$, $A_t$ and $Z_t$. The subproblem is transformed into
\begin{equation}
\begin{aligned}
\underset{P^v}{min}\sum_{v=1}^V\|X^v-P^vAZ\|_F^2
\label{conv2}
\end{aligned}
\end{equation}

Note that if we define $b=\|X_t^v-P_{t+1}^vA_tZ_t\|_{F}^2$, $c=\|X_t^v-P_{t}^vA_tZ_t\|_{F}^2$, their non-negativity is relatively easy to verify. According to Lemma \ref{lemma2}, we can obtain the following result:
\begin{equation}
\begin{aligned}
&\sum_{v=1}^V\|X_t^v-P_{t}^vA_tZ_t\|_{F}^2-\sum_{v=1}^V\frac{\|X_t^v-P_{t}^vA_tZ_t\|_{F}^4}{2\|X_t^v-P_{t}^vA_tZ_t\|_{F}^2} \\
&\ge \sum_{v=1}^V\|X_t^v-P_{t+1}^vA_tZ_t\|_{F}^2-\sum_{v=1}^V\frac{\|X_t^v-P_{t+1}^vA_tZ_t\|_{F}^4}{2\|X_t^v-P_{t}^vA_tZ_t\|_{F}^2}
\label{conv3}
\end{aligned}
\end{equation}

Besides, according to the inexact Majorization-Minimization method \cite{SunBP17,HuNWL20}, we can obtain 
\begin{equation}
\begin{aligned}
 \sum_{v=1}^V\frac{\|X_t^v-P_{t}^vA_tZ_t\|_{F}^4}{2\|X_t^v-P_{t}^vA_tZ_t\|_{F}^2} \ge 
 \sum_{v=1}^V\frac{\|X_t^v-P_{t+1}^vA_tZ_t\|_{F}^4}{2\|X_t^v-P_{t}^vA_tZ_t\|_{F}^2}
\label{conv4}
\end{aligned}
\end{equation}
Combing the Eq.(\ref{conv3}) and Eq.(\ref{conv4}), we have 
\begin{equation}
\begin{aligned}
\sum_{v=1}^V\|X_t^v-P_{t}^vA_tZ_t\|_{F}^2 \ge 
\sum_{v=1}^V\|X_t^v-P_{t+1}^vA_tZ_t\|_{F}^2
\label{conv5}
\end{aligned}
\end{equation}

Eq.(\ref{conv5}) illustrates that $\{\mathbf{P^v}\}_{v=1}^V$ is monotonically non-increasing when the other variables are fixed, therefore we can derive 
\begin{equation}
\begin{aligned}
&\Gamma(\{\mathbf{P^v_t}\}_{v=1}^V, \{\mathbf{W^v_t}\}_{v=1}^V,A_t,Z_t) \ge \\
&\Gamma(\{\mathbf{P^v_{t+1}}\}_{v=1}^V, \{\mathbf{W^v_t}\}_{v=1}^V,A_t,Z_t)
\label{conv6}
\end{aligned}
\end{equation}

2) Update the $\{\mathbf{W_{t+1}^v}\}_{v=1}^V$ while $\{\mathbf{P_t}^v\}_{v=1}^V$, $A_t$ and $Z_t$ are fixed, the subproblem of $\{\mathbf{W^v}\}_{v=1}^V$ is as follows:
\begin{equation}
\begin{aligned}
\min_{W^v}\lambda_1\sum_{v=1}^V\|W^vX^v-AZ\|_F^2 +\lambda_3\|W^v\|_{2,1}
\label{conv7}
\end{aligned}
\end{equation}

Thanks to the literature \cite{NieHCD10,YangYCMFYYL20}, we can obtain the following Theorem \ref{theorem2}:
\begin{theorem}
Considering a general $\ell_{2,1}$-norm minimization problem as:
\begin{equation}
\begin{aligned}
\min_{W}f(W)+\|W\|_{2,1} \ s.t. U\in \mathcal{C}
\end{aligned}
\label{2,1 norm}
\end{equation}
\label{theorem2}
If $f(W)$ is a convex function and $\mathcal{C}$ is a convex set, then the iterative method will converge to the global minimum.
\end{theorem}
Obviously, $f(W)$ is a convex function due to the properties of the Frobenius norm. And the monotonically decreasing proof of the $\mathcal{L}_{2,1}$ norm can be obtained from \cite{NieHCD10}. Thus 
\begin{equation}
\begin{aligned}
&\Gamma(\{\mathbf{P^v_t}\}_{v=1}^V, \{\mathbf{W^v_t}\}_{v=1}^V,A_t,Z_t) \ge \\
&\Gamma(\{\mathbf{P^v_{t}}\}_{v=1}^V, \{\mathbf{W^v_{t+1}}\}_{v=1}^V,A_t,Z_t)
\label{conv8}
\end{aligned}
\end{equation}

3) Update the $A_t$ with fixed other variables. The solution of $A_t$ is
\begin{equation}
\begin{aligned}
\min_{A}\sum_{v=1}^V\|X^v-P^vAZ\|_F^2+\lambda_1\sum_{v=1}^V\|W^vX^v-AZ\|_F^2
\label{conv9}
\end{aligned}
\end{equation}

Since the Frobenius norm is convex \cite{boyd2004convex}, then we can obtain
\begin{equation}
\begin{aligned}
&\Gamma(\{\mathbf{P^v_t}\}_{v=1}^V, \{\mathbf{W^v_t}\}_{v=1}^V,A_t,Z_t) \ge \\
&\Gamma(\{\mathbf{P^v_{t}}\}_{v=1}^V, \{\mathbf{W^v_{t}}\}_{v=1}^V,A_{t+1},Z_{t})
\label{conv10}
\end{aligned}
\end{equation}

4) Update the $Z_{t+1}$ with fixed $\{\mathbf{P_t}^v\}_{v=1}^V$, $\{\mathbf{W_{t+1}^v}\}_{v=1}^V$ and $A_t$. This subproblem can be written as 
\begin{equation}
\begin{aligned}
\min_{Z}\sum_{v=1}^V\|X^v-P^vAZ\|_F^2+\lambda_1\sum_{v=1}^V\|W^vX^v-AZ\|_F^2
\label{conv11}
\end{aligned}
\end{equation}

We can solve the subproblem Eq.(\ref{conv8}) by Lemma \ref{lemma2}, and thus obtain
\begin{equation}
\begin{aligned}
&\Gamma(\{\mathbf{P^v_t}\}_{v=1}^V, \{\mathbf{W^v_t}\}_{v=1}^V,A_t,Z_t) \ge \\
&\Gamma(\{\mathbf{P^v_{t}}\}_{v=1}^V, \{\mathbf{W^v_{t}}\}_{v=1}^V,A_t,Z_{t+1})
\label{conv12}
\end{aligned}
\end{equation}

Finally, joint with Eq.(\ref{conv6}), Eq.(\ref{conv8}), Eq.(\ref{conv10}) and Eq.(\ref{conv12}), we have 
\begin{equation}
\begin{aligned}
&\Gamma(\{\mathbf{P^v_t}\}_{v=1}^V, \{\mathbf{W^v_t}\}_{v=1}^V,A_t,Z_t) \ge \\
&\Gamma(\{\mathbf{P^v_{t+1}}\}_{v=1}^V, \{\mathbf{W^v_{t+1}}\}_{v=1}^V,A_{t+1},Z_{t+1})
\label{conv13}
\end{aligned}
\end{equation}
\end{proof}

Eq.(\ref{conv13}) indicates that the value of the objective function will be nonincreasing in each iteration. Therefore theorem \ref{theorem1} is proved. Furthermore, since the objective function defined in Eq.(\ref{conv1}) has a lower bound of zero \cite{FPMVS}, which also verifies the convergence of our algorithm.

\bibliographystyle{IEEEtran}
\bibliography{main.bib}

\begin{thebibliography}{10}
\providecommand{\url}[1]{#1}
\csname url@samestyle\endcsname
\providecommand{\newblock}{\relax}
\providecommand{\bibinfo}[2]{#2}
\providecommand{\BIBentrySTDinterwordspacing}{\spaceskip=0pt\relax}
\providecommand{\BIBentryALTinterwordstretchfactor}{4}
\providecommand{\BIBentryALTinterwordspacing}{\spaceskip=\fontdimen2\font plus
\BIBentryALTinterwordstretchfactor\fontdimen3\font minus \fontdimen4\font\relax}
\providecommand{\BIBforeignlanguage}[2]{{%
\expandafter\ifx\csname l@#1\endcsname\relax
\typeout{** WARNING: IEEEtran.bst: No hyphenation pattern has been}%
\typeout{** loaded for the language `#1'. Using the pattern for}%
\typeout{** the default language instead.}%
\else
\language=\csname l@#1\endcsname
\fi
#2}}
\providecommand{\BIBdecl}{\relax}
\BIBdecl

\bibitem{ShiNWL23}
S.~Shi, F.~Nie, R.~Wang, and X.~Li, ``Fast multi-view clustering via prototype graph,'' \emph{{IEEE} Trans. Knowl. Data Eng.}, vol.~35, no.~1, pp. 443--455, 2023.

\bibitem{LuYL16}
C.~Lu, S.~Yan, Z.~Lin \emph{et~al.}, ``Convex sparse spectral clustering: Single-view to multi-view,'' \emph{{IEEE} Trans. Image Process.}, vol.~25, no.~6, pp. 2833--2843, 2016.

\bibitem{TrostenLJK21}
D.~J. Trosten, S.~L{\o}kse, R.~Jenssen, and M.~Kampffmeyer, ``Reconsidering representation alignment for multi-view clustering,'' in \emph{CVPR}, 2021, pp. 1255--1265.

\bibitem{Wen0X0HF023}
J.~Wen, C.~Liu, G.~Xu, Z.~Wu, C.~Huang, L.~Fei, and Y.~Xu, ``Highly confident local structure based consensus graph learning for incomplete multi-view clustering,'' in \emph{CVPR}, 2023, pp. 15\,712--15\,721.

\bibitem{YanZLT0LL23}
W.~Yan, Y.~Zhang, C.~Lv, C.~Tang, G.~Yue, L.~Liao, and W.~Lin, ``Gcfagg: Global and cross-view feature aggregation for multi-view clustering,'' in \emph{CVPR}, 2023, pp. 19\,863--19\,872.

\bibitem{CaoZFLZ15}
X.~Cao, C.~Zhang, H.~Fu, S.~Liu, and H.~Zhang, ``Diversity-induced multi-view subspace clustering,'' in \emph{CVPR}, 2015, pp. 586--594.

\bibitem{ShiNWL22}
S.~Shi, F.~Nie, R.~Wang, and X.~Li, ``Self-weighting multi-view spectral clustering based on nuclear norm,'' \emph{Pattern Recognit.}, vol. 124, p. 108429, 2022.

\bibitem{ZhangHFZC17}
C.~Zhang, Q.~Hu, H.~Fu, P.~Zhu, and X.~Cao, ``Latent multi-view subspace clustering,'' in \emph{CVPR}, 2017, pp. 4333--4341.

\bibitem{GMC}
H.~Wang, Y.~Yang, B.~Liu \emph{et~al.}, ``Gmc: Graph-based multi-view clustering,'' \emph{{IEEE} Trans. Knowl. Data Eng.}, vol.~32, no.~6, pp. 1116--1129, 2019.

\bibitem{TanLHF023}
Y.~Tan, Y.~Liu, S.~Huang, W.~Feng, and J.~Lv, ``Sample-level multi-view graph clustering,'' in \emph{CVPR}, 2023, pp. 23\,966--23\,975.

\bibitem{ZhanNCNZY19}
K.~Zhan, C.~Niu, C.~Chen, F.~Nie, C.~Zhang, and Y.~Yang, ``Graph structure fusion for multiview clustering,'' \emph{{IEEE} Trans. Knowl. Data Eng.}, vol.~31, no.~10, pp. 1984--1993, 2019.

\bibitem{GaoHLW13}
J.~Gao, J.~Han, J.~Liu, and C.~Wang, ``Multi-view clustering via joint nonnegative matrix factorization,'' in \emph{SIAM ICDM}, 2013, pp. 252--260.

\bibitem{zhao2017multi}
H.~Zhao, Z.~Ding, Y.~Fu \emph{et~al.}, ``Multi-view clustering via deep matrix factorization,'' in \emph{AAAI}, 2017, p. 2921–2927.

\bibitem{ChenLW0L22}
M.~Chen, T.~Liu, C.~Wang, D.~Huang, and J.~Lai, ``Adaptively-weighted integral space for fast multiview clustering,'' in \emph{ACM MM}, 2022, pp. 3774--3782.

\bibitem{LiuW00LZG22}
S.~Liu, S.~Wang, P.~Zhang, K.~Xu, X.~Liu, C.~Zhang, and F.~Gao, ``Efficient one-pass multi-view subspace clustering with consensus anchors,'' in \emph{AAAI}, 2022, pp. 7576--7584.

\bibitem{Qiang00021}
Q.~Qiang, B.~Zhang, F.~Wang, and F.~Nie, ``Fast multi-view discrete clustering with anchor graphs,'' in \emph{AAAI}, 2021, pp. 9360--9367.

\bibitem{WangLLJTZZ22}
S.~Wang, X.~Liu, S.~Liu, J.~Jin, W.~Tu, X.~Zhu, and E.~Zhu, ``Align then fusion: Generalized large-scale multi-view clustering with anchor matching correspondences,'' in \emph{NeurIPS}, 2022.

\bibitem{WangLWZZH15}
Y.~Wang, X.~Lin, L.~Wu, W.~Zhang, Q.~Zhang, and X.~Huang, ``Robust subspace clustering for multi-view data by exploiting correlation consensus,'' \emph{{IEEE} Trans. Image Process.}, vol.~24, no.~11, pp. 3939--3949, 2015.

\bibitem{YangZNWYW21}
B.~Yang, X.~Zhang, F.~Nie, F.~Wang, W.~Yu, and R.~Wang, ``Fast multi-view clustering via nonnegative and orthogonal factorization,'' \emph{{IEEE} Trans. Image Process.}, vol.~30, pp. 2575--2586, 2021.

\bibitem{SMVSC}
M.~Sun, P.~Zhang, S.~Wang, S.~Zhou, W.~Tu, X.~Liu, E.~Zhu, and C.~Wang, ``Scalable multi-view subspace clustering with unified anchors,'' in \emph{ACM MM}, 2021, pp. 3528--3536.

\bibitem{FPMVS}
S.~Wang, X.~Liu, X.~Zhu, P.~Zhang, Y.~Zhang, F.~Gao, and E.~Zhu, ``Fast parameter-free multi-view subspace clustering with consensus anchor guidance,'' \emph{{IEEE} Trans. Image Process.}, vol.~31, pp. 556--568, 2022.

\bibitem{OMSC}
M.-S. Chen, C.-D. Wang, D.~Huang, J.-H. Lai, and P.~S. Yu, ``Efficient orthogonal multi-view subspace clustering,'' in \emph{ACM SIGKDD}, 2022, pp. 127--135.

\bibitem{HuLY22}
S.~Hu, Z.~Lou, Y.~Ye \emph{et~al.}, ``View-wise versus cluster-wise weight: Which is better for multi-view clustering?'' \emph{{IEEE} Trans. Image Process.}, vol.~31, pp. 58--71, 2022.

\bibitem{chen2023fast}
Z.~Chen, X.-J. Wu, T.~Xu, and J.~Kittler, ``Fast self-guided multi-view subspace clustering,'' \emph{{IEEE} Trans. Image Process.}, 2023.

\bibitem{OPMC}
J.~Liu, X.~Liu, Y.~Yang, L.~Liu, S.~Wang, W.~Liang, and J.~Shi, ``One-pass multi-view clustering for large-scale data,'' in \emph{ICCV}, 2021, pp. 12\,344--12\,353.

\bibitem{XieYYLY18}
L.~Xie, M.~Yin, X.~Yin, Y.~Liu, and G.~Yin, ``Low-rank sparse preserving projections for dimensionality reduction,'' \emph{{IEEE} Trans. Image Process.}, vol.~27, no.~11, pp. 5261--5274, 2018.

\bibitem{WenHFFYZ19}
J.~Wen, N.~Han, X.~Fang, L.~Fei, K.~Yan, and S.~Zhan, ``Low-rank preserving projection via graph regularized reconstruction,'' \emph{{IEEE} Trans. Cybern.}, vol.~49, no.~4, pp. 1279--1291, 2019.

\bibitem{AMGL}
F.~Nie, J.~Li, and X.~Li, ``Parameter-free auto-weighted multiple graph learning: a framework for multiview clustering and semi-supervised classification.'' in \emph{IJCAI}, 2016, pp. 1881--1887.

\bibitem{SFMC}
X.~Li, H.~Zhang, R.~Wang, and F.~Nie, ``Multiview clustering: {A} scalable and parameter-free bipartite graph fusion method,'' \emph{{IEEE} Trans. Pattern Anal. Mach. Intell.}, vol.~44, no.~1, pp. 330--344, 2022.

\bibitem{LMVSC}
Z.~Kang, W.~Zhou, Z.~Zhao, J.~Shao, M.~Han, and Z.~Xu, ``Large-scale multi-view subspace clustering in linear time,'' in \emph{AAAI}, 2020, pp. 4412--4419.

\bibitem{AWMVC}
X.~Wan, X.~Liu, J.~Liu, S.~Wang, Y.~Wen, W.~Liang, E.~Zhu, Z.~Liu, and L.~Zhou, ``Auto-weighted multi-view clustering for large-scale data,'' in \emph{AAAI}, 2023, pp. 10\,078--10\,086.

\bibitem{NieHCD10}
F.~Nie, H.~Huang, X.~Cai, and C.~H.~Q. Ding, ``Efficient and robust feature selection via joint $\ell_{2,1}$-norms minimization,'' in \emph{NeurIPS}, 2010, pp. 1813--1821.

\bibitem{Fu0CZW21}
Z.~Fu, Y.~Zhao, D.~Chang, X.~Zhang, and Y.~Wang, ``Double low-rank representation with projection distance penalty for clustering,'' in \emph{CVPR}, 2021, pp. 5320--5329.

\bibitem{FengLXY14}
J.~Feng, Z.~Lin, H.~Xu, and S.~Yan, ``Robust subspace segmentation with block-diagonal prior,'' in \emph{CVPR}, 2014, pp. 3818--3825.

\bibitem{SunBP17}
Y.~Sun, P.~Babu, and D.~P. Palomar, ``Majorization-minimization algorithms in signal processing, communications, and machine learning,'' \emph{{IEEE} Trans. Signal Process.}, vol.~65, no.~3, pp. 794--816, 2017.

\bibitem{HuNWL20}
Z.~Hu, F.~Nie, R.~Wang, and X.~Li, ``Multi-view spectral clustering via integrating nonnegative embedding and spectral embedding,'' \emph{Inf. Fusion}, vol.~55, pp. 251--259, 2020.

\bibitem{YangYCMFYYL20}
Z.~Yang, Q.~Ye, Q.~Chen, X.~Ma, L.~Fu, G.~Yang, H.~Yan, and F.~Liu, ``Robust discriminant feature selection via joint $\ell_{2,1}$-norm distance minimization and maximization,'' \emph{Knowl. Based Syst.}, vol. 207, p. 106090, 2020.

\bibitem{boyd2004convex}
S.~P. Boyd and L.~Vandenberghe, \emph{Convex optimization}.\hskip 1em plus 0.5em minus 0.4em\relax Cambridge university press, 2004.

\end{thebibliography}

\vfill

\end{document}